\documentclass[12pt]{article}
\usepackage{amsmath}
\usepackage{graphicx}
\usepackage{enumerate}
\usepackage{natbib}
\usepackage{enumitem}
\usepackage{url} 
\RequirePackage{amsthm,amsmath,amsfonts,amssymb}
\usepackage{booktabs} 
\RequirePackage[colorlinks,citecolor=blue,urlcolor=blue]{hyperref}
\RequirePackage{graphicx}
\RequirePackage{algorithm}
\RequirePackage{algorithmic}
\usepackage[title]{appendix}

\usepackage{setspace}
\let\Algorithm\algorithm
\renewcommand\algorithm[1][]{\Algorithm[#1]\setstretch{1.2}}

\usepackage{bbm}
\usepackage{xr}

\newtheorem{theorem}{Theorem}

\newtheorem{proposition}{Proposition}

\theoremstyle{remark}

\newcommand{\blind}{1}

\date{} 
\begin{document}

\def\spacingset#1{\renewcommand{\baselinestretch}%
{#1}\small\normalsize} \spacingset{1}


\if1\blind
{
  \title{\bf Who is the Winning Algorithm?\\Rank Aggregation for Comparative Studies}
  \author{Amichai Painsky\hspace{.2cm}\\
    Tel Aviv University, Israel}
  \maketitle
} \fi

\if0\blind
{
  \bigskip
  \bigskip
  \bigskip
  \begin{center}
    {\LARGE\bf Confidence Intervals for Unobserved Events}
\end{center}
  \medskip
} \fi

\bigskip

\begin{abstract}
Consider a collection of $m$ competing machine learning algorithms. Given their performance on a benchmark of datasets, we would like to identify the best performing algorithm. Specifically, which algorithm is most likely to ``win'' (rank highest) on a future, unseen dataset.  The standard maximum likelihood approach suggests counting the number of wins per each algorithm. In this work, we argue that there is much more information in the complete rankings. That is, the number of times that each algorithm finished second, third and so forth. Yet, it is not entirely clear how to  effectively utilize this information for our purpose. In this work we introduce a novel conceptual framework for estimating the win probability for each of the $m$ algorithms, given their complete rankings over a benchmark of datasets. Our proposed framework significantly improves upon currently known methods in synthetic and real-world examples. 

\end{abstract}

\noindent%
\vfill

\newpage
\spacingset{1.5} 

\section{Introduction}\label{intro}

Introducing a new machine learning (ML) algorithm is a  complex and demanding task. It often begins with a novel concept, a fresh perspective on an existing problem, or insights borrowed from other disciplines. The development process typically involves theoretical analysis, computational considerations, and extensive implementation work. Once the research and development phase is complete, we aim to assess how well the new algorithm performs compared to existing methods. In the ML community, this evaluation is usually done by testing the algorithm on a suite of benchmark datasets and comparing its results to those of state-of-the-art approaches. However, such comparisons rarely yield clear conclusions. Ideally, a strong algorithm should achieve high rankings across many datasets.  But how do we confidently conclude that it outperforms others? Consider, for example, three algorithms evaluated on five datasets. Suppose that one algorithm ranks first on one dataset and second on the remaining four. Another algorithm wins two datasets but ranks third in the rest. Which algorithm is better? Shall we just count the number of wins? Consider the average rank? Or perhaps present the full ranking and leave the interpretation to the reader?

In this work, we address the fundamental problem of identifying the best-performing algorithm based on  benchmark results. We argue that a natural measure of interest is the probability of winning a future dataset. For each algorithm, our goal is to estimate the \textit{win probability} -- the likelihood that it will outperform its competitors in a new predictive modeling task. This measure is particularly interesting in real-world applications, where the objective is to choose the most promising algorithm for deployment. While this problem has been previously explored, as discussed in Section \ref{related work}, a common approach based on maximum likelihood estimation (MLE), which essentially amounts to counting the number of dataset wins for each algorithm. However, we argue that there is much more information in the complete rankings. The task of effectively leveraging this information falls under the broader scope of rank aggregation, a topic that has received notable attention in the ML literature. Nonetheless, existing methods are often heuristic in nature or rely on strong probabilistic assumptions. We show that accurately estimating win probabilities from full ranking data is a challenging and non-trivial task.

In this work,  we introduce a conceptual framework for estimating the win probability in comparative studies. We propose a simple and intuitive scheme that expresses the estimated win probability as a weighted linear combination of the rankings an algorithm achieves across benchmark datasets. The key challenge lies in finding the optimal weights which minimize the estimation error, measured in terms of total variation (TV) distance and Kullback-Leibler (KL) divergence. Naturally, our estimator generalizes the MLE, which corresponds to the special case where all weight is placed solely on dataset wins. We propose two weighting strategies: the first is a conservative, data-independent scheme, where weights are fixed in advance and designed to control worst-case estimation error. The second is a flexible, data-dependent scheme, where weights are adapted based on the observed data. We evaluate both schemes through synthetic and real-world experiments, demonstrating their effectiveness in generating more accurate and reliable comparisons among competing algorithms. Notably, our method leads to clearer and more robust conclusions about algorithmic performance in comparative studies. An implementation of our framework is publicly available on the author's 
webpage\footnote{\url{https://anonymous.4open.science/r/WhoIstheBestAlgorithm-6780/}}.

It is important to emphasize the practical contribution of our work. In comparative studies, multiple machine learning algorithms are typically evaluated across a range of datasets. Naturally, data-handling, pre-processing and hyperparameter tuning have a significant impact on reported results. In this study, we implicitly assume that the performance of each algorithm reflects the end-to-end pipeline, according to recommended guidelines. In that sense, these results represent the outcome a practitioner could realistically achieve. Our contribution lies in proposing a novel, statistically established rank aggregation scheme, for estimating the win probabilities based on these results. 

\section{Related Work}\label{related work}

Evaluating multiple algorithms across a collection of datasets is a common practice. Numerous methods have been proposed for this purpose, differing in their objectives, statistical foundations, presentation style (qualitative or quantitative), visualization strategies, and more. In this section, we provide a brief overview of some of the more widely used approaches and examine their strengths and limitations in the context of our objectives.

The \textit{average performance} is a simple averaging of the algorithms' performance over all the benchmark datasets, according to a predetermined measure of interest. For example, the averaged mean square error (MSE) for regression problems. While this measure is simple and intuitive, it fails to capture the differences between the datasets. For example, consider two datasets where the typical MSE in one dataset is of scale of $10^{-1}$, while in the other is of scale of $10^2$. It is quite misleading to directly average them, since the former would be negligible compared to the latter.

To address this issue, the \textit{average ranking} approach considers the mean rank that each algorithm attains across a collection of datasets. In this framework, the best-performing algorithm on a given dataset is assigned rank first, the second-best is ranked second, and so on. Consequently, a lower average rank indicates stronger overall performance across all datasets. This method is widely used for comparing multiple algorithms simultaneously \citep{demvsar2006statistical,fernandez2014we}, largely due to its simplicity and interpretability. However, the average ranking method has notable limitations. One key drawback is its implicit assumption that all rank differences carry equal weight. That is, the gap between ranking first and fifth is treated the same as the gap between ranking 21st and 25th. This assumption is problematic: while a first-versus-fifth difference often signals performance superiority, the difference between 21st and 25th is typically negligible, especially when many algorithms are being compared. To mitigate this, weighted averaging can be used. A classical example is the \textit{Borda Count} method \citep{emerson2013original}, which assigns each rank a weight based on how many ranks fall below it. For example, a first-place finish earns a weight of $m-1$, second place 
$m-2$, and so forth. Although Borda Count is intuitive and easy to implement, it remains a heuristic and lacks statistical performance guarantees.

Statistical comparisons is yet another alternative to assess dominance among multiple algorithms. Here, we define a single or multiple hypotheses and test whether the data agree with it. A common example is pair-wise comparisons. For every pair of algorithms we test the null hypothesis that they perform equally well. We repeat this process for every pair of algorithms and correct for multiplicity \citep{demvsar2006statistical}. Unfortunately, this routine is not always adequate. First, the number of hypotheses grows quadratically with the number of algorithms, which makes the multiplicity correction very conservative. Second, the result may not be clear enough. That is, a typical outcome may be ``Algorithm A is better than C, and B is better than E'', as the rest of comparisons do not yield statistically significant conclusions. Naturally, there exist more sophisticated statistical tests for the problem. The ANOVA is a parametric test \citep{fisher1956statistical}. It defines a null hypothesis that all the algorithms perform the same and the observed differences are merely random. Unfortunately, this approach does not suggest which is the better algorithm. The Friedman test \citep{friedman1937use,friedman1940comparison} is a non-parametric counter-part of the ANOVA. It considers the same null as ANOVA and defines a rank statistic with a known null distribution.  If the null hypothesis is rejected, one may  proceed with a post-hoc test. The Nemenyi test \citep{nemenyi1963distribution} is similar to the Tukey test for ANOVA \citep{tukey1949comparing} and is used when all algorithms are compared to each other. The performance of two algorithms is significantly different if the corresponding average ranks differ by at least a critical difference, defined in \citep{demvsar2006statistical}. There exist additional alternative tests, such as the Bonferroni-Dunn test \citep{dunn1961multiple}, Hommel's procedure \citep{hommel1988stagewise}, Holm's procedure \citep{holm1979simple} and others. The interested reader is referred to Section $3.2$ of Dem{\v{s}}ar \citep{demvsar2006statistical} for a comprehensive discussion. The statistical methods mentioned above differ from one to another in several ways. However, they all have a common advantage. They provide a solid framework for deciding on the dominant algorithms, if such exist. However, these methods do not provide a decisive conclusion in cases where the differences are less evident. Furthermore, seeking statistically significant differences on the given dataset does not necessarily imply that an algorithm would outperform its alternatives on a future dataset. In fact, this is one of the major differences between statistical inference and prediction. We further discuss this issue later in Section \ref{experiments}. It is worth mentioning that statistical comparisons typically accompany other measures. For example, it is very common to report the average ranking together with a statistical test that validates which rankings are significant. The presentation of such results are provided in tabular form or as visual representation such as critical difference diagrams, heat-maps, box-plots and others \citep{ye2024closer}.  

Another approach to assess dominance is the Bradley–Terry model \citep{bradley1952rank}. Here, the outcome of pairwise comparisons between two algorithms, A and B, is modeled by $\mathbb{P}(A>B)=\alpha_A/(\alpha_A+\alpha_B)$ where $A>B$ means that Algorithm A is ``better'' than B, and $\alpha_A$ is a positive real-valued score assigned to Algorithm A. The Bradley–Terry model can be generalizes to multiple competitors such that $\mathbb{P}(A>\{B,C,D\})=\alpha_A/(\alpha_A+\alpha_B+\alpha_C+\alpha_D)$. This model is known in the literature as the Plackett-Luce model (see Section $5$ of \citep{hunter2004mm}). Given the performance on a  benchmark of datasets, we may write a likelihood function according to the model and obtain maximum likelihood estimates for the $\alpha$'s. The Bradley–Terry and the  Plackett-Luce models are very popular in sport tournaments, commercial applications and other domains \citep{hunter2004mm}. Indeed, they strive to provide a direct answer to our problem. Unfortunately, these models assume a very specific probabilistic structure, which does not comply with most setups. For example, consider the simplest case where the performance of the competing algorithms are independent and normally distributed with different means. It is straightforward to show that such distribution cannot be represented by the models above. We demonstrate this caveat later in Section \ref{experiments}.

Recently, Fern{\'a}ndez-Delgado et al. \citep{fernandez2014we} introduced an alternative measure of dominance. The \textit{Probability of Achieving the Best Accuracy} (PAMA) criterion measures the proportion of datasets on which a given algorithm attains the highest accuracy.  Notably, PAMA focuses solely on the top-ranked algorithm for each dataset, implicitly assuming that win probability is determined only by past wins, rather than the complete rankings. Yet, this approach directly assess our objective of interest.  Building on this concept, our work extends the idea by proposing a novel framework that estimates win probability using the complete ranking information across all datasets.

\section{Definitions and Problem Statement}\label{defs}

Consider a collection of $m$ ML algorithms, applied to $n$ independent datasets. Let $X_i$ be a vector of rankings of the $m$ algorithms over the $i^{th}$ dataset. For example, $x_i=[2,3,1,4]$ indicates that Algorithm $2$ achieved the best performance on the $i^{th}$ dataset, followed by Algorithms $3,1,4$ respectively. Naturally, there are $m!$ possible rankings for the $m$ algorithms. We denote the probability vector of all possible rankings as $P$ and assume that $X_i$ are independent and identically distributed observations from $P$. This assumption implies two basic restrictions. First, the rankings of the algorithms are independent between datasets ($X_i$ are independent variables). Second, the underlying probability $P$ is consistent across all datasets  ($X_i$ are identically distributed). Although these assumptions may seem strong, they are commonly used most of the methods discussed in Section \ref{related work}.

A natural objective is to estimate $P$ from a collection of observations $x^n=x_1,...,x_n$. However, this task is both very difficult and unnecessary in practice. First, notice that the alphabet size of $P$ (the number of possible outcomes) is $m!$ which may be much greater than $n$. In fact, comparing ten algorithms we have $10!\approx 3.6M$ possible outcomes (rankings), while a typical sample only consists of tens (or hundreds) of datasets. Second, we are practically not interested in the complete rankings of all algorithms. Our goal is to estimate the probability that the $j^{th}$ algorithm wins a future dataset, for $j=1,...,m$. We denote the win probability vector as $p$. Notice that $p_j$ is the sum of all the rankings in $P$ in which the $j^{th}$ algorithm finishes first. Formally, define 
$\mathcal{C}_j$ as the collection of all possible rankings in which the $j^{th}$ algorithm is top-ranked. Then, $p_j=\sum_{l \in \ \mathcal{C}_j}P_l$. 
Thus, our objective is to estimate the much smaller vector $p$ (of size $m$) rather than the full distribution $P$ (of size $m!$). 
\newpage
In this work we focus on two fundamental divergence measures between $p$ and its estimate $\hat{p}$. Namely, the \textit{total variation} (TV) and the Kullback-Leibler (KL) Divergence, 
\begin{align}\label{divergences}
    D_{\text{TV}}(p,\hat{p})=\sum_{j=1}^m|p_j-\hat{p}_j|,\quad D_{\text{KL}} (p||\hat{p})=\sum_{j=1}^mp_j\log\frac{{p}_j}{\hat{p}_j}
\end{align}
respectively \citep{jiao2015minimax,painsky2018universality}. Since both measures are random (as they depend on $\hat{p}$), we focus on their expected performance, also denoted as risk, $\mathbb{E}D_{\text{TV}}(p,\hat{p})$ and $\mathbb{E}D_{\text{KL}}(p || \hat{p})$. 

A natural estimator for $p$ is of course the maximum likelihood estimator (MLE), $\hat{p}_{ml}$. Given a sample $x_1,...,x_n$, we define its likelihood as $$\mathcal{L}=\frac{n!}{\prod^{m!}_{l=1}r_l!}\prod_{l=1}^{m!} P_l^{r_l} $$ 
where $r_l$ is the number of times that the $l^{th}$ ranking appears in the sample. Notice $\mathcal{L}$ is simply a multinomial likelihood function over $P$, and the corresponding MLE is given by $\hat{P}^{ml}=r/n$. As mentioned above, we are interested in a MLE for $p$. Hence, by the basic properties of the MLE we have that $\hat{p}^{ml}=r^{(1)}/n$ where $r^{(1)}_j$ is the number of times that the $j^{th}$ algorithm is top-ranked. 
In other words, the MLE corresponds to the proportion of datasets each algorithm wins. This is exactly the PAMA criterion,  discussed in Section \ref{defs}.

In this work we seek $\hat{p}$ that minimizes the expected divergence (\ref{divergences}). Our point of departure is the MLE, which only considers $r^{(1)}$, the number of wins per each algorithm. We argue that there may be more valuable information about $p$ in additional statistics, such as the number of times that each algorithm finishes second, third and so forth. Hence, we propose the following estimator
\begin{align}\label{our_model}
    \hat{p}^w=\sum_{j=1}^m \frac{w_j}{n} r^{(j)}
\end{align}
where $r^{(j)}$ corresponds to the number of times each algorithm is ranked in the $j^{th}$ positon and $w_j$ are non-negative weights satisfying $\sum_j w_j=1 $.  In other words, this model generalizes the MLE (which corresponds to the special case where $w_1=1$) by allowing the weight to be distributed across all ranks. Notice that $\hat{p}^w_j\geq 0$ and $\sum_j\hat{p}^w_j=1$ for every vector of $w$'s defined above, as desired. Hence, our goal is to minimize the expected divergence (\ref{divergences}) over a choice of $w$. We introduce two complementary strategies to the problem. The first is a worst-case, data-independent strategy, where the weights $w$ are fixed prior to observing any data.  The second is a data-dependent approach, in which the weights are determined adaptively based on the observed sample.

\section{The Data Independent Scheme}\label{data independent}

We begin our analysis by considering the expected total variation under our proposed model with fixed weights. Specifically, we would like to minimize  $\mathbb{E}D_{\text{TV}}(p||\hat{p}^w)$ with respect to a fixed vector $w$. Notice that the main difficulty with this approach is that $p$ is unknown in practice. This means that the obtained weights would also depend on the unknown $p$. Therefore, we take a more conservative minimax approach and consider a universal scheme that minimizes the expected divergence for the worst-case scenario. Formally, define $\Delta_{m!}$ as the collection of all possible distributions over an alphabet size $m!$. Further, let $\mathcal{W}$ be the collection of all $w$ such that  $w_j\geq 0$ for $j=1,...,m$ and $\sum_{j=1}^m w_j=1$. Then, our goal is 
\begin{align} 
    \label{minimax TV}
            &\min_{w \in \mathcal{W}} \max_{P\in \Delta_{m!}} \;\; \mathbb{E}D_{\text{TV}}\left(p,\hat{p}^w\right).
\end{align}
Unfortunately, this minimax  problem in (\ref{minimax TV}) is not an easy task. Notice that for a choice of $w_1=1$, it degenerates to  $$\max_{p\in \Delta_m} \mathbb{E}D_{\text{TV}}\left(p,\hat{p}^{ml}\right)$$ which corresponds to the worst-case convergence rate of the MLE. This problem was extensively studied over the years, with several key results. In fact, it was shown that 
\begin{align}\label{MLE minimax}
\max_{p\in \Delta_m} \mathbb{E}D_{\text{TV}}\left(p,\hat{p}^{ml}\right)\leq \sqrt{\frac{m}{n}}
\end{align}
where the bound is asymptotically tight under a uniform distribution $p=1/m$ \citep{jiao2015minimax}. Following these important contributions, we first introduce an upper bound to the worst-case expected divergence, and then minimize it with respect to $w \in \mathcal{W}$. Finally, we discuss the tightness of our results. For the simplicity of the presentation, we first focus on a simpler special case where $\hat{p}^w=(wr^{(1)}+(1-w)r^{(2)})/n$. That is, we only consider the number of times that every algorithm is top-ranked, $r^{(1)}$, or runner-up, $r^{(2)}$, over the given testbench. Theorem \ref{T1}, whose proof is located in Section \ref{proof1} , introduces our proposed upper bound for this setup.
\begin{theorem}\label{T1}
Let $p \in \Delta_m$ be the win probability for each of the algorithms. Let $\hat{p}^w=(wr^{(1)}+(1-w)r^{(2)})/n$ be our proposed minimax estimator. Then, 
    \begin{align}\label{ub_T1}
\max_{P\in \Delta_{m!}}\mathbb{E}D_{\text{TV}}\left(p,\hat{p}^{w}\right)\leq \sqrt{m}\sqrt{2(1-w)^2+\frac{1}{n}(w^2+(1-w)^2)}.
\end{align}
Further, the weight $w^*$ which minimizes (\ref{ub_T1}) is given by $$w^*=1-\frac{1}{2n+2}.$$
\end{theorem}

Let us examine the obtained result. First, we observe that the minimax weight,  $w^*=1-1/(2n+2)$ approaches $1$ for sufficiently large $n$. It is also relatively close to it, even for smaller $n$. Further, the proposed upper bound (\ref{ub_T1}) is almost the same as the MLE's (\ref{MLE minimax}) for a choice of $w^*$. This is not quite surprising. For example, consider a case where one of the algorithms is always ranked second (with probability $1$), while the rest of the algorithms are equally likely to finish first. In this worst-case scenario, putting weight on $r^{(2)}$ is counter-productive with respect to estimating $p$. Furthermore, the estimator's worst-case performance is governed by the (almost) uniform nature of $p$, leading to an (almost) similar bound as the MLE. In other words, we construct a distribution $p$ for which the proposed bound in (\ref{ub_T1}) becomes tight, as desired.  Despite the above, the proposed minimax scheme with $w^*$ may be useful in cases where $n$ is relatively small and a conservative approach is required. We further discuss this scenario in Section \ref{experiments}. It is important to stress that the worst-case distribution that nearly achieves the bound in (\ref{ub_T1}) is highly unlikely in practical scenarios. That is, $w^*$ controls a worst-case ``pathological'' $P$, where one algorithm always ranks second while the rest are equally likely to be first.  In real-world applications, such behavior is improbable, making the use of $w^*$  overly conservative. On the other hand, imposing additional constraints on $P$ to overcome this issue requires additional assumptions on the underlying model. 

To further refine and generalize our results, we introduce Theorem \ref{T3} which extends the analysis to the top $K$ ranks, $r^{(1)},...,r^{(K)}$. This theorem provides an upper bound that can be efficiently computed through a straightforward convex optimization over $K$ parameters.  The full proof of Theorem \ref{T3} is provided in Section \ref{proof2}.
 
\begin{theorem}\label{T3}
Let $p \in \Delta_m$ be the probability that each algorithm wins. Let $\hat{p}^w=\frac{1}{n}\sum_{j=1}^K w_jr^{(j)}$ be our proposed estimator. Assume that $w_1\geq w_2\geq...\geq w_K$ and $w_1\geq \frac{8n-\sqrt{8n}}{8n-1}\approx1-{1}/{\sqrt{8n}}$. Then, 
    \begin{align}\nonumber
&\mathbb{E}D_{\text{TV}}\left(p,\hat{p}^{w}\right)\leq \max_{t=[t_1,...t_K]^T\in\Delta_K} \sum_{j=1}^{K-1} \sqrt{\left((1-w_1)t_j-w_{j+1}\right)^2+\frac{1}{n}\left( w_1^2t_j+ w_2^2\right)}+\\\nonumber
   &\quad\quad\quad\quad\quad\quad\quad\quad\quad\quad\sqrt{(1-w_1)^2t_K^2+\frac{(m-K+1)}{n}\left( w_1^2t_K+ w_2^2(m-K+1)\right)}
\end{align}
Further, under the stated conditions, the maximization problem above is convex in $t=[t_1,...t_K]^T\in\Delta_K$.
\end{theorem}

Theorem \ref{T3} provides a worst-case upper bound on the expected TV  between $p$ and its corresponding linear estimator that incorporates the top $K$ ranks. While this theorem enhances model flexibility, it does not yield a closed-form expression for the optimal weights $w^*$ that minimize the bound. In practice, this means we must numerically evaluate various weight vectors $w$ that satisfy the theorem’s conditions and select the one which minimizes the bound.

Despite this limitation, Theorem \ref{T3} introduces a notable improvement over Theorem \ref{T1}, as it allows for a richer model that leverages more ranking information.
To illustrate this, we compare Theorem \ref{T3} for $K=2$ and $K=3$ with Theorem \ref{T1} and the MLE bound (\ref{MLE minimax}). Figure \ref{fig21} presents the obtained upper bounds for $m=5$ and varying values of $n$. As we can see, the $K=2$ bound closely matches Theorem \ref{T1}, with minor differences attributed to its numerical nature.  Yet, we observe a notable improvement in our results as the model increases to $K=3$, demonstrating the benefit of incorporating additional rank information.


These results should not be interpreted as establishing the optimality of our proposed scheme. Rather, they demonstrate that the proposed upper bound strictly improves upon the best currently known result (\ref{MLE minimax}). In other words, although (\ref{MLE minimax}) is known to be asymptotically tight \citep{jiao2015minimax}, its finite-sample behavior is subject to further study. Consequently, our contribution does not constitute an improvement over a provably tight MLE finite-sample result. Instead, we improve the strongest available upper bound to date, yielding refined performance guarantees. In the following sections, we complement our theoretical findings with a detailed study of the proposed bound, evaluated on both synthetic and real-world examples.

\begin{figure}[ht]
\centering
\includegraphics[width =0.6\textwidth,bb=100 260 480 570,clip]{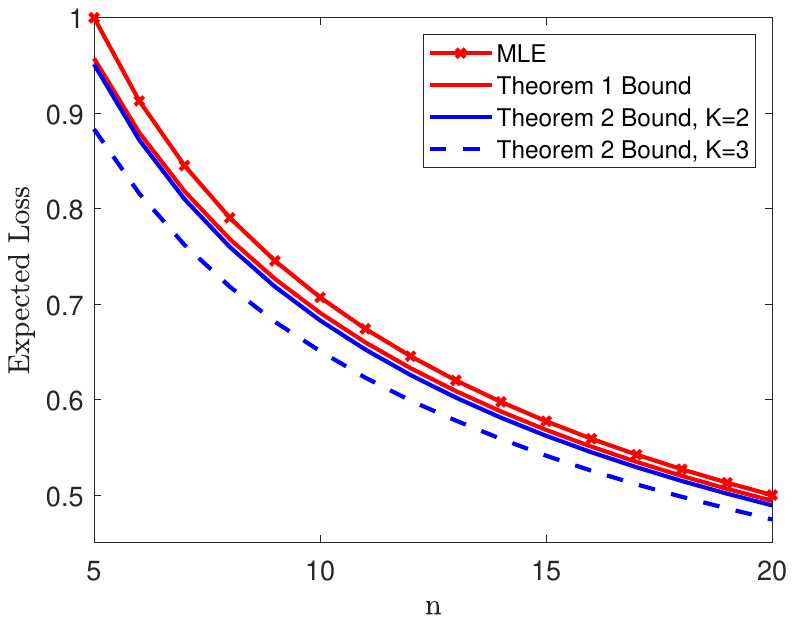}
\caption{Data-independent upper bounds for $m=5$, under the expected TV distance}
\label{fig21}
\end{figure}


Next, we would like to consider a similar minimax approach for the expected KL divergence (\ref{divergences}). Unfortunately, this task is quite problematic. Notice that the KL divergence is  unbounded in cases where $p_j>0$ while $\hat{p}_j=0$. This means that any estimator that assigns zero probabilities is prone to obtain  unbounded loss (and hence unbounded expected loss). For this reason, minimax analysis of discrete distributions under KL divergence require additional assumptions and constraints, especially in a finite sample regime. See \citep{orlitsky2015competitive} for a comprehensive discussion. Therefore, we refrain from a worst-case KL analysis and study this divergence only in a data-dependent regime. 

To conclude, the minimax approach provides data-independent weights $w$ while considering the worst-case distribution $P\in \Delta_{m!}$. While inherently conservative, this strategy produces robust and non-trivial weights that offer a meaningful improvement over the classical MLE, especially in small-sample settings (see Section \ref{experiments}). In the following section we introduce an alternative data-dependent scheme which does not consider a worst-case distribution and obtains favorable weights which depend on the data in hand.

\section{The Data Dependent Scheme}\label{data dependent}
Ideally, we would like to minimize the expected loss, where the expectation is taken over the true underlying model $P$. Unfortunately, this model is unknown, and so we require an alternative approach. In Section \ref{data independent} we focus on a minimax criterion, where we consider the worst-case underlying model. Naturally, this leads to conservative results. In this section we consider an empirical risk minimization scheme, which replaces the expectation with its empirical counter-part. This approach is perhaps the most common paradigm in learning theory, with many favorable properties. We begin our analysis with the KL divergence, and later introduce a complementary TV scheme. First, we observe that minimizing $\mathbb{E}D_{\text{kl}}(p||\hat{p})$ with respect to $\hat{p}$ is equivalent to minimizing $-\mathbb{E}\sum_{j=1}^m p_j\log\hat{p}_j$. Hence, the corresponding empirical risk is given by 
\begin{align}
    R^{\text{KL}}=-\frac{1}{n}\sum_{i=1}^n\sum_{j=1}^m\mathbb{I}(y_i=j)\log(\hat{p}_j)=-\frac{1}{n}\sum_{j=1}^m r^{(1)}_j\log \hat{p}_j 
\end{align}
where $y_i$ denotes the winning algorithm over the $i^{th}$ dataset. This is exactly the well-known cross-entropy loss function. Plugging our proposed model (\ref{our_model}) to the above and minimizing over $w$ yields $w_1=1$, which corresponds to the MLE.  This is not surprising as the MLE is known to be the minimizer of the cross-entropy function (over any choice of $\hat{p}$). 

However, it is well known that directly minimizing the empirical loss may lead to overfitting the observed data. To address this issue, we adopt the widely used leave-one-out (loo) strategy as an alternative. Given a collection of 
$n$ datasets and $m$ algorithms, the proposed procedure is as follows: in each iteration, we remove one dataset and compute the estimator, $\hat{p}^w$,  based on the $n-1$ remaining datasets. Then, we evaluate the loss on the excluded dataset. This process is repeated for each of the $n$ datasets, and we select the weight vector $w$ that minimizes the average loss across all loo iterations. Formally,  let $x^n=x_1,...,x_n$ be the collection of rankings over all the datasets (as defined in Section \ref{defs}). Denote $x^{n}{[-i]}$ the collection of all the datasets, with the $i^{th}$ dataset removed. Likewise,  define $r^{(j)}[-i]$ as the number of times that each algorithm finished at the $j^{th}$ position, over the collection of datasets $x^{n}{[-i]}$. Then, the corresponding estimator is $\hat{p}^w[-i]=(1/n) \sum_{j=1}^m w_jr^{(j)}[-i]$. This serves as our loo estimator for the $i^{th}$ iteration. Therefore, the loo loss is given by
\begin{align}\label{loo loss}
    L^{\text{KL}}_{loo}=-\frac{1}{n}\sum_{i=1}^n\sum_{j=1}^m\mathbb{I}(y_i=j)\log(\hat{p}^w_j[-i]),
\end{align}
and $w^{loo}$ is the vector of weights $w\in\mathcal{W}$ which minimizes (\ref{loo loss}).  Notice that given $x^n$, the loo loss minimization (\ref{loo loss}) is convex in $w$ and may be evaluated by standard convex optimization. Algorithm \ref{alg1} below summarizes our proposed method.
\begin{algorithm}
\caption{Leave-one-out Estimation Scheme}
\begin{algorithmic} [1]
\REQUIRE $x^n$, $m$
\STATE Define $x^{n}{[-i]}$ as the collection of all the dataset, excluding the $i^{th}$ dataset.
\STATE Define $r^{(j)}[-i]$ as the number of times that each algorithm finished at the $j^{th}$ position, over the collection  $x^{n}{[-i]}$.
\STATE Denote  $\hat{p}^w[-i]=(1/n) \sum_{j=1}^m w_jr^{(j)}[-i]$.
\STATE Minimize (\ref{loo loss}) with respect to $w\in \mathcal{W}$. 
\STATE Define $w^{loo}$ as the minimizer of Step $4$.
\RETURN $\hat{p}^w=\sum_{j=1}^mw^{loo}_jr^{(j)}$. 
\end{algorithmic}
\label{alg1}
\end{algorithm}

As opposed to the minimax approach, the proposed loo scheme introduces data-dependent weights, which are less conservative and consider the data in hand. Notice that for every fixed $w$, the loo loss is an (almost) unbiased estimate of the expected loss $\mathbb{E}D_{\text{KL}}(p||\hat{p}_w)$. Hence, we obtain an (almost) unbiased estimate for our desired measure of interest. The loo approach is not new to learning theory and practice. In fact, is also quite common in highly related parameter estimation problems. For example, the seminal Good-Turing estimator \citep{good1953population}, which is perhaps the most common scheme for distribution estimation over large alphabets (that is, where $m\gg n$) can also be interpreted as a loo scheme \citep{mcallester2000convergence}. In that sense, our proposed method utilizes a similar conceptual framework of parameter estimation in a discrete and relatively large alphabet regime, as later discussed in Section \ref{experiments}. While Algorithm \ref{alg1} allows minimizing the loo loss with respect to $w_j$, where $j=1,...,m$, we typically restrict our attention to simpler models where $j\leq 3$. This corresponds to a model which only considers the top three ranks in each dataset (similarly to $K=3$ in Theorem \ref{T3}). We observe that for higher order models, the weights are typically negligible. In addition, we add a constraint to step $4$ in Algorithm \ref{alg1}, which  requires that $w^{loo}_j$ are non increasing with $j$. This constraint is quite natural in practice, as $r^{(j)}$ should be more informative about the winning algorithm than $r^{(j+1)}$.

Finally, we note that a similar loo approach applies to any divergence measure of interest. For example, the empirical counter-part of the TV is known to be $R^{TV}=\frac{1}{n}\sum_{i=1}^n\sum_{j=1}^m|\mathbb{I}(y_i=j)-\hat{p}_j|$. Hence, the corresponding loo scheme is given by    
\begin{align}\label{TV loo loss}
    R^{TV}_{loo}=\frac{1}{n}\sum_{i=1}^n\sum_{j=1}^m|\mathbb{I}(y_i=j)-\hat{p}^w_j[-i]|,
\end{align}
which is again convex in $w\in\mathcal{W}$ and holds all the favorable properties of  (\ref{loo loss}). We demonstrate, compare  and discuss both the proposed estimators in the following section. 

\section{Experiments}\label{experiments}

Let us now illustrate the performance of our suggested scheme in synthetic and real-world experiments. 

\subsection{Synthetic Experiments}
First, we study study six example distributions, which are common benchmarks for probability estimation and related problems \citep{orlitsky2015competitive}. The Zipf's law distribution is a typical benchmark in large alphabet probability estimation; it is a commonly used heavy-tailed distribution, mostly for modeling natural (real-world) quantities in physical and social sciences, linguistics, economics and others fields \citep{saichev2009theory}. The Zipf's law distribution follows $P(u;s,m)={u^{-s}}/{\sum_{v=1}^m v^{-s}}$ where $s$ is a skewness parameter.  Additional examples of commonly used heavy-tailed distributions are the geometric distribution, $P(u;\alpha)=(1-\alpha)^{u-1}\alpha$, the negative-binomial distribution (specifically, see \citep{efron1976estimating}), $P(u;l,r)= \binom{u+l-1}{u} r^u(1-r)^l$ and the beta-binomial distribution $P(u;m,\alpha,\beta)= \binom{m}{u}{B(u+\alpha,m-u+\beta)}/{B(\alpha,\beta)}$. Notice that the support of the geometric and the negative-binomial distributions is infinite. Therefore, for the purpose of our experiments, we truncate them to an alphabet size $m$ and normalize accordingly. Additional example distributions are the uniform, $p(u)=1/m$, and the step distribution, $P(u)\propto 1$ for $u\leq m/2$ and $P(u)\propto 1/2$ otherwise. 
In our context, the heavy-tailed  distributions correspond to a case where a small number of algorithms outperform the others, while the more uniform setups correspond to the case where all algorithms are more equally ranked.


In our first experiment we study $m=6$ competing algorithms. Therefore, we have $m!=720$ possible rankings and $P$ is the probability of each of these rankings. The ordering of the alphabet (that is, which ranking corresponds to which probability) is assigned at random. We draw $n$ samples from $P$. Notice that the $i^{th}$ sample corresponds to a complete ranking of the $m$ algorithms on $i^{th}$ dataset. Given the $n$ samples, we estimate $p$ and evaluate the divergence of interest (\ref{divergences}). We repeat this process $1000$ times (that is, draw $n$ samples, compute $\hat{p}$ and evaluate $D_{\text{KL}}(p||\hat{p}),\; D_{\text{TV}}(p,\hat{p})$) to attain the averaged divergence. We compare the following estimators. First, we compute the MLE as a benchmark method. Then, we evaluate our proposed data-dependent scheme (Algorithm \ref{alg1}) with only three weights $w_1,w_2,w_3$, as discussed above. We further compute the data-independent scheme where relevant. As a lower bound, we evaluate an Oracle estimator who knows the true $p$ and computes the optimal weights by minimizing the true expected divergence. Last, we consider an additional alternative, as we directly estimate $P$ with Good-Turing large alphabet estimator \citep{gale1995good,painsky2022convergence}, and obtain a corresponding estimation for $p$. Figure \ref{fig1} demonstrates the results we achieve under KL divergence, for the six distributions above, as $n$ increases. 

\begin{figure}[ht]
\centering
\includegraphics[width =0.9\textwidth,bb=30 90 770 510,clip]{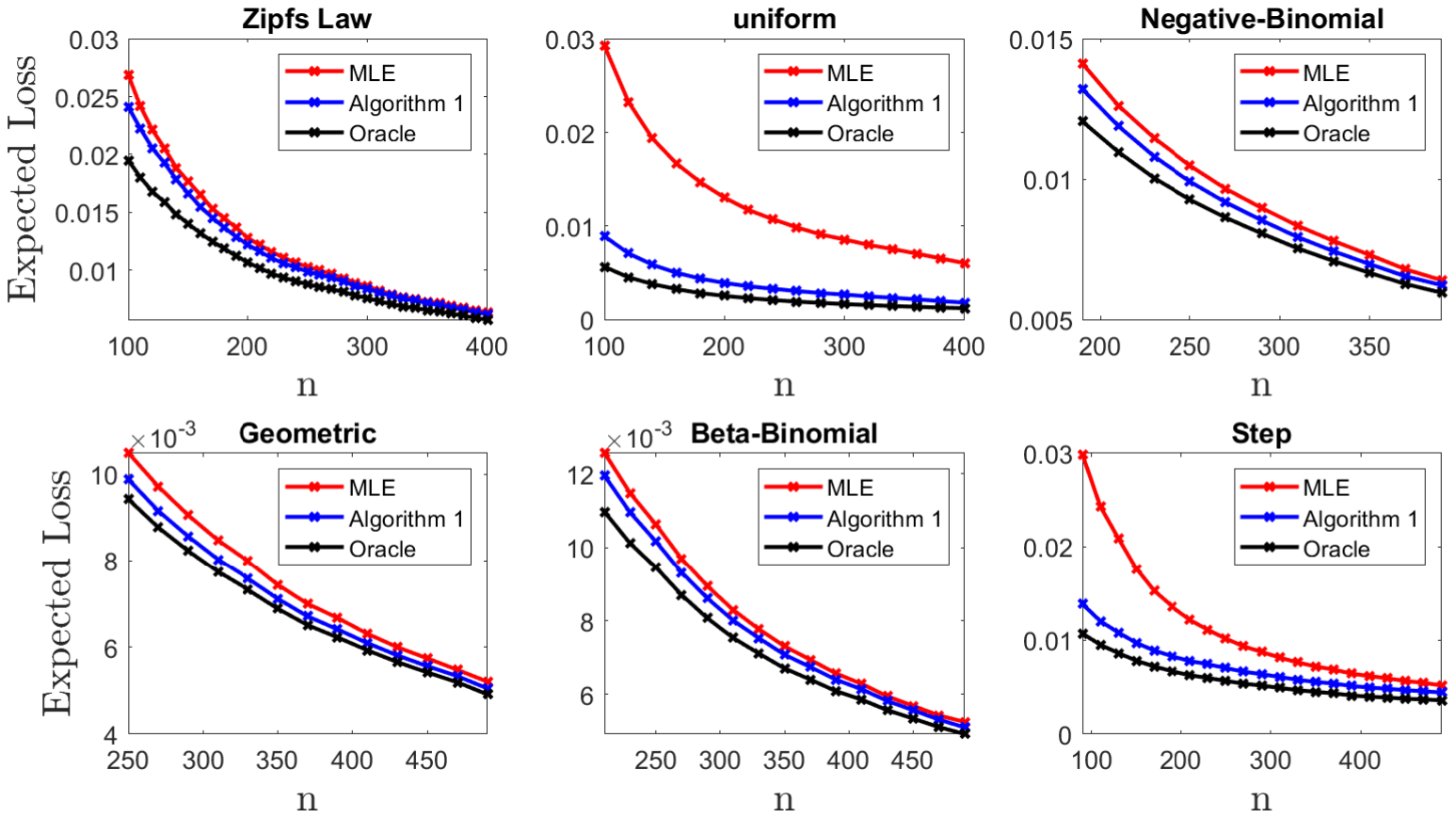}
\caption{Comparing Algorithm $1$ (blue) with the MLE (red) and an Oracle estimator (black) in two synthetic experiments. We use the following parameters: Zipf's Law: $s=1.01$, Geometric: $
\alpha=0.4$, Negative-Binomial: $l=1, r=0.003$, Beta-Binomial: $\alpha=\beta=2$.}
\label{fig1}
\end{figure}

As we can see, our proposed scheme (in blue) outperforms the MLE (in red) quite significantly. This is especially evident with the uniform and step distributions, where it is only natural to put more weight on $r^{(2)}$ and $r^{(3)}$, as opposed to completely ignoring them. Further, our proposed scheme is relatively close to the Oracle, where the gap closes as $n$ increases. This is not quite surprisingly, as Algorithm \ref{alg1} is a leave-one-out, which performs better as $n$ increases. Finally, notice that all three methods converge to the same performance for sufficiently large $n$, as expected. We omit the GT approach from Figure \ref{fig1} as it is not competitive for with the alternatives. Let us further examine the results of Algorithm \ref{alg1}. Figure \ref{fig2} presents the average weights of Algorithm \ref{alg1}, compared with the Orcale weights.  We focus our attention to the Zipf's Law and the uniform distributions for brevity. First, consider the Zipf's Law experiment (left). We notice that $w^{loo}_1$ tends to $1$ as $n$ increases. On the other hand, the weights are much more uniform (and steady) with the uniform distribution (right). Again, this is quite obvious. The Zipf's Law distribution is heavy-tailed, which means that a few algorithms are dominant, This means we put more weight of the winner. On the other hand, the uniform distribution suggests that all algorithms are equally likely to win, which means that $r^{(2)}$ and $r^{(3)}$ are not negligible, compared to $r^{(1)}$. Our proposed scheme successfully adapts this behavior, given only the data in hand.

\begin{figure}[ht]
\centering
\includegraphics[width =0.8\textwidth,bb=10 180 570 650,clip]{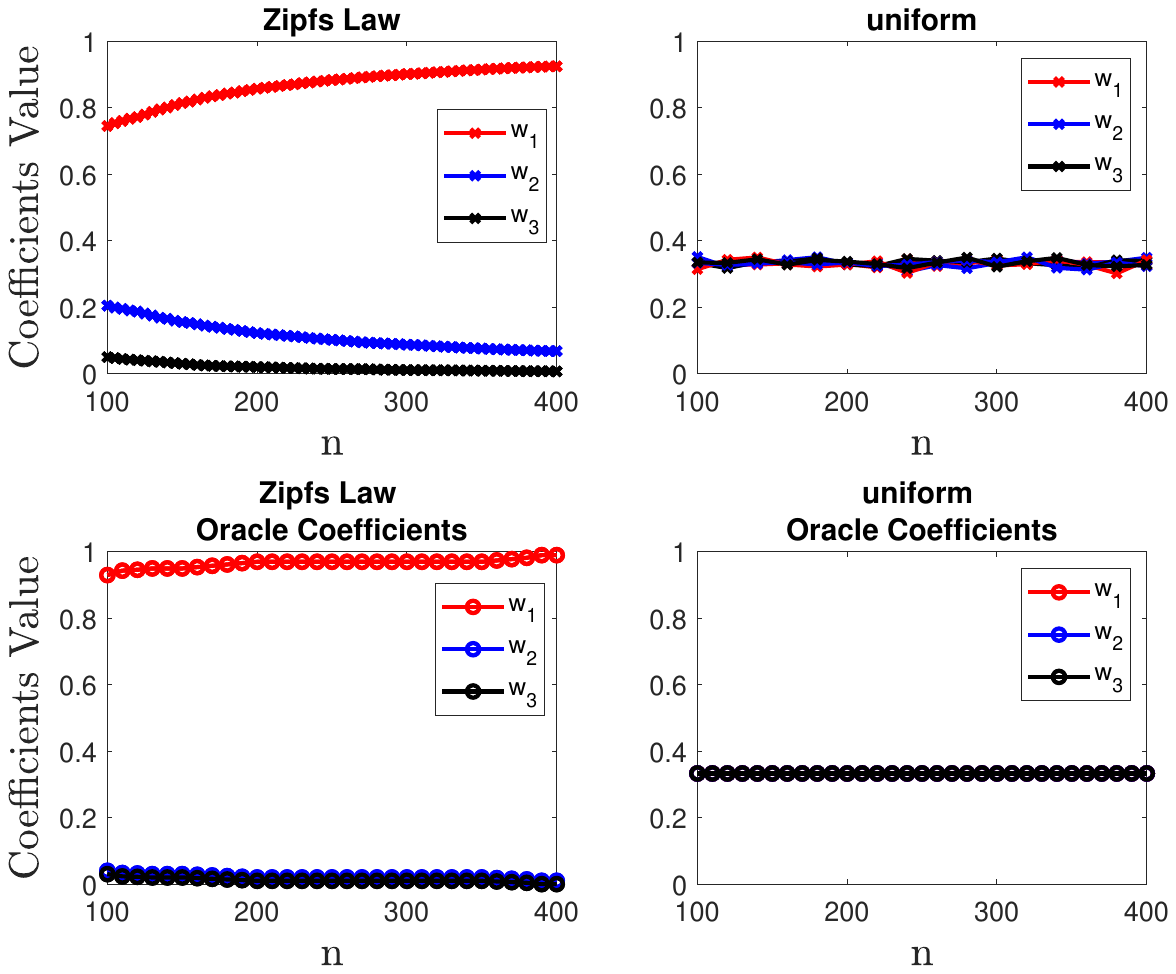}
\caption{The obtained weights of Algorithm $1$ (upper charts) and the Oracle (lower charts) in the synthetic experiment above (Figure \ref{fig1})}
\label{fig2}
\end{figure}

Next, we repeat the same experiment under the TV distance. Here, we focus our attention to a very small number of datasets (small $n$), which is also quite common in the literature. We illustrate our minimax approach (Section \ref{data independent}), which is better suited for such limited-sample setups than the leave-one-out method. Figure \ref{fig3} demonstrates the results we achieve for Theorem \ref{T1} and Theorem \ref{T3} with $K=3$. Once again, we compare our results to an Oracle who knows $P$ but restricted to a $K=2$ model. 
The observed trends closely resemble those from the previous experiment: our proposed estimators consistently outperform alternative approaches (with GT omitted again due to poor performance). Notably, both minimax-based methods significantly outperform the MLE, as we observe a notable improvement when extending the model from $K=2$ to $K=3$.

\begin{figure}[ht]
\centering
\includegraphics[width =1\textwidth,bb=40 50 770 540,clip]{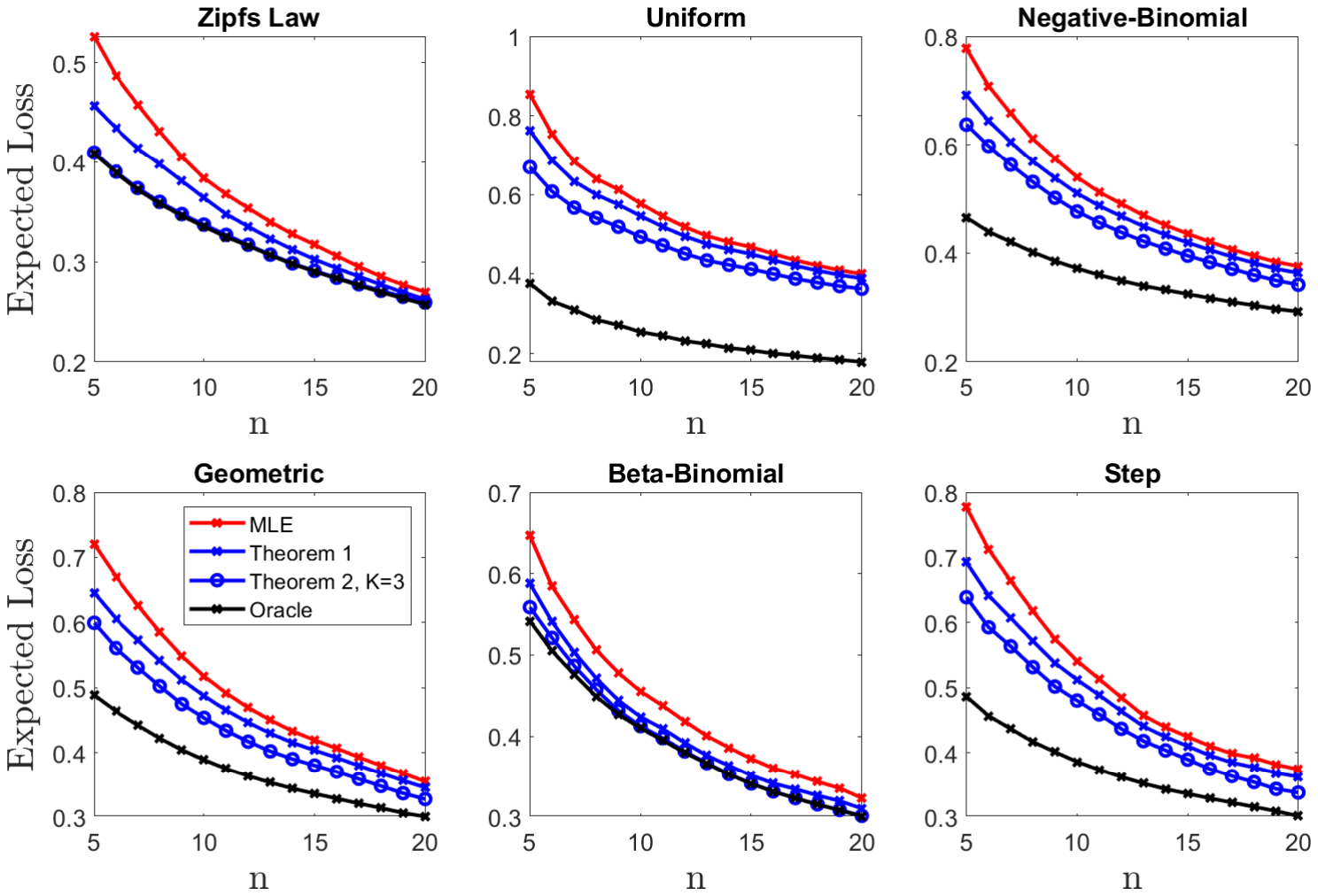}
\caption{Comparing the data-independent schemes (Theorems \ref{T1} and \ref{T3}) with the MLE (red) and an Oracle estimator (black) in small sample  synthetic experiments. The distributions' parameters are described in Figure \ref{fig1}. }
\label{fig3}
\vspace{-1.5em}
\end{figure}

\subsection{Real-world Experiments}

We now proceed to real-world experiments. Here, we consider large comparative studies, where multiple key ML algorithms are compared against each other over a benchmark of datasets. First, we revisit the work of McElfresh et al. \citep{mcelfresh2024neural}. In this work, the authors introduce a benchmark of $176$ tabular datasets, and compare a total of $19$ popular ML algorithms. Their goal is to study the performance of gradient boosted decision trees (GBDT) compared to neural networks (NN), over tabular data. They conclude that ``for a surprisingly high number of datasets, either the performance difference between GBDTs and NNs is negligible, or light hyper-parameter tuning on a GBDT is more important than choosing between NNs and GBDTs''. In addition, they introduce the \textit{TabZilla} benchmark suite, which is a collection of $36$ of their ``hardest'' datasets. The performance of the studied algorithms over TabZilla is provided in Table $4$ of McElfresh et al. \citep{mcelfresh2024neural}, where they report the rankings of the top three algorithms on each dataset. As mentioned above, the conclusions are not decisive. The top-performing algorithms (according to TabZilla) vary quite significantly, depending on the measure of interest. The PAMA criterion (which corresponds to the MLE), suggests that the top three algorithms are \textit{CatBoost}, \textit{XGBoost} and \textit{LightGBM} with corresponding estimates of $0.2667,0.2667$ and $0.2333$, respectively. Notice that the averaged rankings cannot be evaluated from Table $4$, as it requires the full ranking and not just the top three for each dataset. Applying our suggested data-dependent scheme to Table $4$ (under KL divergence, following Algorithm \ref{alg1}), we obtain almost uniform $w^{loo}$. This leads to the top three algorithms: \textit{XGBoost}, \textit{CatBoost} and \textit{LightGBM} with corresponding estimates of $0.2777,  0.2447$ and $0.1563$, respectively. In other words, while the MLE suggests that only $r^{(1)}$ is an informative statistic, our proposed scheme suggests that there is also information in $r^{(2)},r^{(3)}$. In fact, notice that while \textit{CatBoost} and \textit{XGBoost} attain the same number of wins, they differ quite significantly in the number of times they finish second (\textit{CatBoost} with $5$ times, while \textit{XGBoost} with $16$). Our proposed scheme successfully utilizes this information  and draws a more decisive conclusion. Further, we apply the Borda Count and the Bradley-Terry models (see Section \ref{related work}).  Interestingly, the Borda Count model leads to the following top three algorithms: \textit{CatBoost}, \textit{XGBoost} and \textit{ResNet}. On the other hand, the Bradley-Terry model suggests that the best-performing algorithms are \textit{CatBoost}, \textit{LightGBM} and \textit{ResNet}, where the model's scores are evaluated only according to the available top three ranks.  It is important to note that the additional statistical tests discussed in Section \ref{related work} are not applicable to our comparison. Specifically, our objective is to derive a point estimate for the best-performing algorithm, while statistical tests are designed to assess whether differences in benchmark performance are statistically significant. Therefore, comparing point estimators with hypothesis tests is conceptually misleading, as they serve fundamentally different purposes.

Finally, we conduct the following experiment to validate our results. We define a subset of datasets  on which we evaluate $\hat{p}$ (namely, the ``train-set''). Then, we examine the estimators' performance on the remaining datasets (``test-set''). We perform $k$-fold splitting, and repeat the above for every split. We report the average loss (in terms of cross-entropy). Our proposed method outperforms the MLE with an average loss of $1.78$ compared to $1.86$, with a corresponding p-value of $0.0075$. These results further support the effectiveness of our approach. The Borda Count attains an average loss of $1.91$ while the Bradley-Terry model achieves an average loss of $1.97$. Table \ref{tab:real_datasets} summerizes our results.

\begin{table}[h]
\centering
\caption{\textit{TabZilla} benchmark results \citep{mcelfresh2024neural}. We report the three top-performing algorithms according to each scheme with a corresponding average cross-validation loss, as described in the text. The statistically significant best performing scheme (at level $0.05$) is \textbf{bolded}.}
\label{tab:real_datasets}
\begin{tabular}{lrrrr}
\toprule
\textbf{Scheme} & \textbf{First} & \textbf{Second} & \textbf{Third} & \textbf{Average Loss} \\
\midrule
MLE & \textit{CatBoost} & \textit{XGBoost} & \textit{LightGBM} &  $1.86$ \\
Bradley-Terry & \textit{CatBoost} & \textit{LightGBM} & \textit{ResNet} &  $1.97$ \\
Borda Count & \textit{CatBoost} & \textit{XGBoost} & \textit{ResNet} &  $1.91$ \\
Our proposed method & \textit{XGBoost} & \textit{CatBoost} & \textit{LightGBM} &  \textbf{1.78} \\
\bottomrule
\end{tabular}
\end{table}

Next, we focus on the work Fern{\'a}ndez-Delgado et al. \citep{fernandez2014we}. In this work, the authors evaluate a total of  $179$ classifiers (arising from $17$ basic models), implemented on Weka, R, C and Matlab. They use $121$ data sets, which represent the entire UCI repository and additional real problems, in order to achieve significant conclusions about the classifiers' behavior. They conclude that the classifiers most likely to be the best are the random forest (RF) versions. However, the gap with the second best (SVM with a Gaussian kernel) is not statistically significant. We apply our proposed scheme to the results of Fern{\'a}ndez-Delgado at al. \citep{fernandez2014we}. First, notice that the study of Fern{\'a}ndez-Delgado et al. \citep{fernandez2014we} compares an extremely large number of algorithms, which even exceeds the number of datasets. Furthermore, many of the algorithms are practically identical, as they only differ in the platform that they were implemented on. Hence, analyzing the winner among this set of algorithms is very noisy and prone to random results. Therefore, as a first step in reducing unnecessary comparisons, we only study one version of each algorithm. That is, we do not compare the same algorithm, implemented on two different platforms. For simplicity, let us focus on the Weka algorithms, which consist of $m=51$. The top performing Weka algorithms according to the MLE criterion are \textit{Rotation Forest}, \textit{Random Committee}, \textit{END}, \textit{LibLINEAR} and \textit{Decorate} with estimated probabilities of $0.1368$, $0.0684$, $0.0684$, $0.0598$ and  $0.0598$, respectively. Applying Algorithm \ref{alg1} we conclude that the top algorithms are \textit{Rotation Forest}, \textit{Random Committee}, \textit{Decorate}, \textit{END} and \textit{Classification Via Regression} with estimated probabilities of  $0.1033$, $0.0689$, $0.0684$, $0.0536$ and $0.0534$, respectively. Importantly, $w^{loo}=0.48, 0.33, 0.19$. This suggests that the differences between the algorithms are relatively small. Our proposed scheme identifies almost the same top-performing algorithms as the MLE. However, both the order and the estimated probabilities vary. The minimum average rank criterion suggests that the top performing algorithms are \textit{Rotation Forest}, \textit{Classification Via Regression}, \textit{Random Committee}, \textit{Decorate} and \textit{Simple Logistic}. Notice that the minimum average rank of the best performing algorithm according to this criterion is $10.08$, which only emphasizes the noisy nature of this comparison. The Borda Count scheme suggests that the best performing algorithms are \textit{Rotation Forest},\textit{LibLINEAR}, \textit{END}, \textit{Decorate} and \textit{Classification Via Regression}. On the other hand, the Bradley-Terry results shows that the best performing algorithm is \textit{Rotation Forest}, followed by \textit{Classification Via Regression}, \textit{LibLINEAR},\textit{Random Committee} and  \textit{END}. Finally, to validate our results we again perform a $k$-fold cross validation as described above. Once again, our proposed scheme attains better results with $3.32$ compared to $3.34$ of the MLE. The corresponding p-value is $0.07$. Naturally, the difference here is less significant due to the noisy comparison. Further, the Borda Count attains an average error of $3.46$ while the Bradley-Terry model is even less competitive with an average error of $3.51$. Table \ref{tab:real_datasets_2} summerizes our results. 

\begin{table}[h]
\centering
\caption{Fern{\'a}ndez-Delgado et al comparative study \citep{fernandez2014we}. We report the three top-performing algorithms according to each scheme with a corresponding average cross-validation loss, as described in the text. The statistically significant best performing scheme (at level $0.1$) is \textbf{bolded}.}
\label{tab:real_datasets_2}
\begin{tabular}{llllr}
\toprule
\textbf{Scheme} & \textbf{First} & \textbf{Second} & \textbf{Third} &  \textbf{Average Loss} \\
\midrule
MLE & \textit{\begin{tabular}{l}Rotation\\Forest\end{tabular}}& \textit{\begin{tabular}{l}Random\\Committee\end{tabular}}& \textit{\begin{tabular}{l}END\end{tabular}}& $3.34$ \\
Bradley-Terry 
& \textit{\begin{tabular}{l}Rotation\\Forest\end{tabular}} & \textit{\begin{tabular}{l}Classification\\via Regression\end{tabular}} & \textit{\begin{tabular}{l}LibLINEAR\end{tabular}} &  $3.51$ \\
Borda Count & \textit{\begin{tabular}{l}Rotation\\Forest\end{tabular}} &\textit{\begin{tabular}{l}LibLINEAR\end{tabular}}& \textit{\begin{tabular}{l}END\end{tabular}}& $3.46$ \\
Our proposed method & \textit{\begin{tabular}{l}Rotation\\Forest\end{tabular}}& \textit{\begin{tabular}{l}Random\\Committee\end{tabular}}& \textit{\begin{tabular}{l}Decorate\end{tabular}}&\textbf{3.32} \\
\bottomrule
\end{tabular}
\end{table}

Finally, we turn to the comparative study of  Shmuel et al. \citep{shmuel2024comprehensive}. In this work, the authors introduce a comprehensive benchmark, aimed at better characterizing the types of datasets where deep learning models excel. They evaluate $111$ datasets with $20$ different models, including both regression and classification tasks. These datasets vary in scale and include both those with and without categorical variables. Applying the ML criterion, the top three algorithms are \textit{AutoGluon}, \textit{SVM} and \textit{ResNet} with corresponding estimated probabilities of $0.3514$, $0.0901$ and $0.0811$, respectively. Our proposed scheme suggests that the leading algorithms are  \textit{AutoGluon}, \textit{CatBoost} and \textit{LightGBM} with corresponding estimated probabilities of $0.2787$, $0.0814$ and $0.0780$, respectively. Interestingly, the minimum average rank criterion agrees with our proposed scheme and recommends the same top algorithms. The Borda Count criterion leads to \textit{AutoGluon}, \textit{ResNet} and \textit{LightGBM} as the top-performing models while Bradley-Terry suggests \textit{CatBoost}, \textit{AutoGluon} and \textit{ResNet}.
Evaluating our results using $k$-fold cross validation  as described above, we again observe a slight improvement with an average loss of $2.42$, compared to the alternative MLE with $2.45$, where the corresponding p-value is $0.01$. Once again, the Borda Count and the Bradley-Terry model demonstrate inferior results. The Borda Count attains an average loss of $2.59$ while the Bradley-Terry model achieves an average loss of $2.72$. Table \ref{tab:real_datasets_3} summarizes our results. 

\begin{table}[h]
\centering
\caption{Shmuel et al. study \cite{shmuel2024comprehensive}. We report the three top-performing algorithms according to each scheme with a corresponding average cross-validation loss, as described in the text. The statistically significant best performing scheme (at level $0.05$) is \textbf{bolded}.}
\label{tab:real_datasets_3}
\begin{tabular}{lrrrr}
\toprule
\textbf{Scheme} & \textbf{First} & \textbf{Second} & \textbf{Third} & \textbf{Average Loss} \\
\midrule
MLE & \textit{AutoGluon} & \textit{SVM} & \textit{ResNet} &  $2.45$ \\
Bradley-Terry & \textit{CatBoost} & \textit{AutoGluon} & \textit{ResNet} &  $2.72$ \\
Borda Count & \textit{AutoGluon} & \textit{ResNet} & \textit{LightGBM} &  $2.59$ \\
Our proposed method &  \textit{AutoGluon} & \textit{CatBoost} & \textit{LightGBM} &  \textbf{2.42} \\
\bottomrule
\end{tabular}
\end{table}

\section{Proofs}
\subsection{A proof for Theorem \ref{T1}}\label{proof1}
First, we have
\begin{align}
\mathbb{E}D_{\text{TV}}\left(p,\hat{p}^{w}\right)=&\mathbb{E}\sum_{j=1}^m|p_j-\hat{p}_j^w|=\mathbb{E}\sum_{j=1}^m\sqrt{\left(p_j-\hat{p}_j^w\right)^2}\leq\sum_{j=1}^m\sqrt{\mathbb{E}\left(p_j-\hat{p}_j^w\right)^2}
\end{align}
where the last inequality follows from Jensen inequality. Denote $q \in \Delta_m$ as a distribution vector which corresponds to the probability that an algorithm finishes second. That is $q_j$ is the probability that the $j^{th}$ algorithm is the second ranked algorithm over a given dataset. From the classical variance-bias decomposition, we attain
$$\mathbb{E}\left(p_j-\hat{p}_j^w\right)^2=\left(p_j-\mathbb{E}(\hat{p}_j^w)\right)^2+\text{var}\big(\hat{p}_j^w\big),$$
where 
$$\mathbb{E}\big(\hat{p}_j^w\big)=\mathbb{E}\big(wr_j^{(1)}+(1-w)r_j^{(2)}\big)/n=wp_j+(1-w)q_j$$
and 
\begin{align}
&\text{var}\big(\hat{p}_j^w\big)=\\\nonumber
&\frac{w^2}{n^2}\text{var}\big(r_j^{(1)}\big)+\frac{(1-w)^2}{n^2}\text{var}\big(r_j^{(2)}\big)+\frac{2w(1-w)}{n^2}\text{cov}\big(r_j^{(1)},r_j^{(2)}\big)=\\\nonumber
&w^2\frac{p_j(1-p_j)}{n}+(1-w)^2\frac{q_j(1-q_j)}{n}-2w(1-w)\frac{p_jq_j}{n}. 
\end{align}
Putting together the above, we obtain
\begin{align}
&\mathbb{E}\left(p_j-\hat{p}_j^w\right)^2=\\\nonumber
&p_j^2-2p_j(wp_j+(1-w)q_j)+(wp_j+(1-w)q_j)^2+\\\nonumber
&\frac{1}{n}\left(w^2p_j+(1-w)^2q_j-(wp_j+(1-w)q_j)^2\right)=\\\nonumber
&p_j^2-2p_j(wp_j+(1-w)q_j)+\left(1-\frac{1}{n}\right)(wp_j+(1-w)q_j)^2+\\\nonumber
&\frac{1}{n}\left(w^2p_j+(1-w)^2q_j\right)<\\\nonumber
&(p_j-wp_j-(1-w)q_j))^2+\frac{1}{n}\left(w^2p_j+(1-w)^2q_j\right)=\\\nonumber
&(1-w)^2(p_j-q_j)^2+\frac{1}{n}\left(w^2p_j+(1-w)^2q_j\right)
\end{align}
where the inequality follows from $1-1/n<1$. Hence, 
\begin{align}
&\mathbb{E}D_{\text{TV}}\left(p,\hat{p}^{w}\right)\leq\\\nonumber
& \sum_{j=1}^m\sqrt{(1-w)^2(p_j-q_j)^2+\frac{1}{n}\left(w^2p_j+(1-w)^2q_j\right)}\leq\\\nonumber
&\max_{P\in\Delta_{m!}} \sum_{j=1}^m\sqrt{(1-w)^2(p_j-q_j)^2+\frac{1}{n}\left(w^2p_j+(1-w)^2q_j\right)}.
\end{align}
Denote $t_j=(1-w)^2(p_j-q_j)^2+\frac{1}{n}\left(w^2p_j+(1-w)^2q_j\right).$ Notice we have that 
\begin{align}\label{const2}
    &0\leq t_j \leq (1-w)^2+\frac{1}{n}(w^2+(1-w)^2)\triangleq g_1(w)\\\label{const1}
    &0\leq \sum_{j=1}^mt_j\leq2(1-w)^2+\frac{1}{n}(w^2+(1-w)^2)\triangleq g_2(w)
\end{align}
where both upper bounds are attained for  $p_u=1$ and $q_v=1$ for some $u\neq v$. 
Putting it together, we would like to bound from above $\sum_{j=1}^n\sqrt{t_j}$ subject to (\ref{const1}) and (\ref{const2}). This is a concave maximization problem over a convex set, and the maximum is attain for $t_j^*=g_2(w)/m$. Hence, we get that
\begin{align}\label{ub_final}\nonumber
\mathbb{E}D_{\text{TV}}\left(p,\hat{p}^{w}\right)\leq \sqrt{m\cdot g_2(w)}=\sqrt{m}\sqrt{2(1-w)^2+\frac{1}{n}(w^2+(1-w)^2)}.
\end{align}
Notice that for $w=1$ we obtain  $\mathbb{E}D_{\text{TV}}\left(p,\hat{p}^{w}\right)\leq \sqrt{m/n}$ as desired (\ref{MLE minimax}).
Finally, we would like to minimize the above with respect to $w$. Simple derivation shows that $w^*=1-{1}/{(2n+2)}$. 

\subsection{A proof for Theorem \ref{T3}} \label{proof2}
First, we have
\begin{align}
        \mathbb{E}D_{\text{TV}}\left(p,\hat{p}^{w}\right)=\mathbb{E}\sum_{j=1}^m|p_j-\hat{p}_j^w|=\mathbb{E}\sum_{j=1}^m\sqrt{\left(p_j-\hat{p}_j^w\right)^2}\leq\sum_{j=1}^m\sqrt{\mathbb{E}\left(p_j-\hat{p}_j^w\right)^2}
\end{align}
where the last inequality follows from Jensen inequality. Denote $p^{(k)} \in \Delta_m$ as a distribution vector which corresponds to the probability that an algorithm finishes $k^{th}$ place. That is $p^{(k)}_j$ is the probability that the $j^{th}$ algorithm is the $k^{th}$ ranked algorithm over a future dataset. To avoid an overload of notation, we refer to $p^{(1)}$ as $p$ interchangeably. From the classical variance-bias decomposition, we attain
$$\mathbb{E}\left(p_j-\hat{p}_j^w\right)^2=\left(p_j-\mathbb{E}(\hat{p}_j^w)\right)^2+\text{var}(\hat{p}_j^w),$$
where 
$$\mathbb{E}(\hat{p}_j^w)=\frac{1}{n}\mathbb{E}\left(\sum_{k=1}^K w_k r_j^{(k)}\right)=\sum_{k=1}^Kw_k p^{(k)}_j$$
and 
\begin{align}
\text{var}(\hat{p}_j^w)=&\sum_{k=1}^K\frac{w_k^2}{n^2}\text{var}(r_j^{(k)})+\sum_{u \neq v}\frac{2w_uw_v}{n^2}\text{cov}(r_j^{(u)},r_j^{(v)})=\\\nonumber
&\sum_{k=1}^K w_k^2\frac{p^{(k)}_j(1-p^{(k)}_j)}{n}-2\sum_{u\neq v}w_uw_v\frac{p^{(u)}_jp^{(v)}_j}{n}. 
\end{align}
Putting together the above, we obtain
\begin{align}
\mathbb{E}\left(p_j-\hat{p}_j^w\right)^2=&\left((1-w_1)p_j-\sum_{k=2}^K w_kp^{(k)}_j\right)^2+\\\nonumber
&\frac{1}{n}\left(\sum_{k=1}^K w_k^2p^{(k)}_j(1-p^{(k)}_j)-2\sum_{u\neq v}w_uw_vp^{(u)}_jp^{(v)}_j\right)\leq\\\nonumber
&\left((1-w_1)p_j-\sum_{k=2}^K w_kp^{(k)}_j\right)^2+\frac{1}{n}\sum_{k=1}^K w_k^2p^{(k)}_j.
\end{align}
Hence, 
\begin{align}\label{T3_1}
\mathbb{E}D_{\text{TV}}\left(p,\hat{p}^{w}\right)\leq& \sum_{j=1}^m\sqrt{\left((1-w_1)p_j-\sum_{k=2}^K w_kp^{(k)}_j\right)^2+\frac{1}{n}\sum_{k=1}^K w_k^2p^{(k)}_j}.
\end{align}
We now bound (\ref{T3_1}) from above with respect to $p^{(1)},...,p^{(k)}$. 
Let $\mathcal{U}_m$ be a doubly stochastic matrix such that for every $U \in \mathcal{U}_m$ we have $\sum_{i=1}^m U_{ij}=\sum_{j=1}^m U_{ij}=1$ and $U_{ij}\geq 0$. We use $U \in \mathcal{U}_m$ to represent $p^{(1)},...,p^{(m)}$. Specifically, recall that $p^{(k)}_j$ is the probability that the $j^{th}$ algorithms is the $k^{th}$ ranked algorithm.  This means that 
$\sum_{k=1}^m p^{(k)}_j=1$ for every $j=1,...,m$ and $\sum_{j=1}^m p^{(k)}_j=1$ for every $k=1,...,m$. Therefore, we represent $p^{(1)},...,p^{(m)}$ with $U$, such that the $k^{th}$ column of $U$ corresponds to $p^{(k)}$. That is, $U=[p^{(1)},...,p^{(m)}]$. Yet, our optimization problem is only with respect to $p^{(1)},...,p^{(K)}$. Hence, we define a collection $\mathcal{V}_{K,m}$ is the following manner. Let $U\in \mathcal{U}_{m}$ be a doubly stochastic matrix. Define $V$ as a subset of $K$ columns of $U$. Let $\mathcal{V}_{K,m}$ be the collection of all possible $V$ matrices. Naturally, for every $V\in  \mathcal{V}_{K,m}$ we have $\sum_{i=1}^mV_{ij}=1$ for every $j=1,...,K$ but $\sum_{j=1}^K V_{ij}\leq 1$ for every $i=1,...,m$. Therefore, we represent $p^{(1)},...,p^{(K)}$ as the columns of $V$ such that  $V=[p^{(1)},...,p^{(K)}]\in \mathcal{V}_{K,m} $ and obtain
\begin{align}\label{T3_2}
\mathbb{E}D_{\text{TV}}\left(p,\hat{p}^{w}\right)\leq\max_{p^{(1)},...,p^{(K)}\in\mathcal{V}_{K,m}} \sum_{j=1}^mf\left(p^{(1)}_j,...,p^{(k)}_j\right).
\end{align}
where 
$$f\left(p^{(1)}_j,...,p^{(k)}_j\right)=\sqrt{\left((1-w_1)p_j-\sum_{k=2}^K w_kp^{(k)}_j\right)^2+\frac{1}{n}\sum_{k=1}^K w_k^2p^{(k)}_j}.$$ 
Let us now relax the maximization above (\ref{T3_2}) by considering $p^{(1)}$ and $p^{(2)},...,p^{(K)}$  separately. That is, 
\begin{align}
&\max_{p^{(1)},...,p^{(K)}\in\mathcal{V}_{K,m}} \sum_{j=1}^mf\left(p^{(1)}_j,...,p^{(k)}_j\right)\leq \max_{p^{(1)}\in\Delta_{m}}\left(\max_{p^{(2)}...p^{(K)}\in \mathcal{V}_{K-1,m}} \sum_{j=1}^mf\left(p^{(1)}_j,...,p^{(K)}_j\right)\right)\nonumber
\end{align}
In words, we maximize the objective over a marginal $p^{(1)}$ which does not depend on the choice of the remaining distributions. Let us first consider the maximization over $p^{(2)},...,p^{(K)}$. Proposition \ref{prop2} below shows that the maximum of $\sum_jf\left(p^{(1)}_j,...,p^{(K)}_j\right)$ over $p^{(2)},...,p^{(K)}\in \mathcal{V}_{K-1,m}$ is attained on the vertices of the set. That is, for degenerate distributions $p^{(k)}$, $k=2,...,K$. 

Formally, define $\mathcal{T}_m$ as a collection of permutation matrices. That is, every $T\in \mathcal{T}_m$ is a doubly stochastic matrix and $T_{ij}\in\{0,1\}$. Next, define $S$ as a \textit{partial permutation matrix} of $K$ columns. That is, for every $T\in \mathcal{T}_m$ we define $S$ as a subset of $K$ columns of $T$. Let $\mathcal{S}_{K,m}$ be the collection of all partial permutation matrices $S$. For example, assume $m=3$ and $K=2$. Then, 
$\mathcal{T}_m$ is the collection
\begin{align}\nonumber
\left(\begin{array}{ccc} 1 & 0 &0\\ 0 & 1 & 0 \\0 & 0& 1\end{array}\right),
\left(\begin{array}{ccc} 1 & 0 &0\\ 0 & 0 & 1 \\0 &1 & 0\end{array}\right),
\left(\begin{array}{ccc} 0 & 1 &0\\ 1 & 0 & 0 \\0 & 0& 1\end{array}\right), 
\left(\begin{array}{ccc} 0 & 0 &1\\ 1 & 0 & 0 \\0 & 1& 0\end{array}\right), 
\left(\begin{array}{ccc} 0 & 1 &0\\ 0 & 0 & 1 \\1 & 0& 0\end{array}\right), 
\left(\begin{array}{ccc} 0 & 0 &1\\ 0 & 1 & 0 \\1 & 0& 0\end{array}\right) 
\end{align}
and $\mathcal{S}_{K,m}$ is simply
\begin{align}\nonumber
\left(\begin{array}{cc} 1 & 0\\ 0 & 1\\0 & 0\end{array}\right),
\left(\begin{array}{cc} 1 & 0\\ 0 & 0 \\0 &1 \end{array}\right),
\left(\begin{array}{cc} 0 & 1\\ 1 & 0  \\0 & 0\end{array}\right), 
\left(\begin{array}{cc} 0 & 0\\ 1 & 0  \\0 & 1\end{array}\right), 
\left(\begin{array}{cc} 0 & 1\\ 0 & 0  \\1 & 0\end{array}\right), 
\left(\begin{array}{cc} 0 & 0\\ 0 & 1  \\1 & 0\end{array}\right). 
\end{align}
Notice that $|\mathcal{S}_{K,m}|=|\mathcal{T}_m|$. Further, every $S \in \mathcal{S}_{K,m}$ satisfies the following:
\begin{itemize}
    \item $\sum_i S_{ij}=1$ for every $j=1,...,m$.
    \item $\sum_j S_{ij} \in \{0,1\}$ for every $i=1,...,K$.
    \item $S_{ij}\in\{0,1\}$ for every $j=1,...,m$ and $i=1,...,K$.
\end{itemize}
In words, every columns in $S\in S_{K,m}$ corresponds to a degenerate distribution and no two columns are identical.   

\begin{proposition}\label{prop2}
    Let $S_{{K-1},m}$ be a collection of  partial permutation matrices. Assume that $w_2\geq w_3 \geq ...\geq w_K$. Then, 
   $$\max_{p^{(2)}...p^{(K)}\in \mathcal{V}_{K-1,m}} \sum_{j=1}^mf\left(p^{(1)}_j,...,p^{(K)}_j\right)\leq\max_{p^{(2)}...p^{(K)}\in \mathcal{S}_{K-1,m}} \sum_{j=1}^mg\left(p^{(1)}_j,...,p^{(K)}_j\right)$$
   where 
   $$g\left(p^{(1)}_j,...,p^{(K)}_j\right)=
   \sqrt{\left((1-w_1)p^{(1)}_j-\sum_{k=2}^K w_kp^{(k)}_j\right)^2+\frac{1}{n}\left( w_1^2p_j^{(1)}+\frac{1}{n} w_2^2\right)}$$
\end{proposition}
\begin{proof}
    First, we have that \begin{align}
    f\left(p^{(1)}_j,...,p^{(K)}_j\right)=&\sqrt{\left((1-w_1)p_j^{(1)}-\sum_{k=2}^K w_kp^{(k)}_j\right)^2+\frac{1}{n}\sum_{k=1}^K w_k^2p^{(k)}_j}\leq\\\nonumber
    &\sqrt{\left((1-w_1)p^{(1)}_j-\sum_{k=2}^K w_kp^{(k)}_j\right)^2+\frac{1}{n}\left( w_1^2p_j^{(1)}+ w_2^2\right)}\triangleq\\\nonumber
    &g\left(p^{(1)}_j,...,p^{(K)}_j\right).
    \end{align} 
    where the inequality is due to $$\sum_{k=1}^K w_k^2p^{(k)}_j=w_1^2p^{(1)}_j+\sum_{k=2}^K w_k^2p^{(k)}_j\leq w_1^2p^{(1)}_j +\max_{k=2,...,K}w_k^2=w_1^2p^{(1)}_j +w_2^2$$
    as $w_2\geq w_3\geq...\geq w_K$ and $\sum_{k=2}^K p_j^{(k)}\leq1$. 
    Next, it is immediate to show that the Hessian of $g\left(p^{(1)}_j,...,p^{(K)}_j\right)$ with respect to $p^{(2)}_j,...,p^{(K)}_j$ is positive semi-definite. This means that $g\left(p^{(1)}_j,...,p^{(K)}_j\right)$ is convex with respect to $p^{(2)}_j,...,p^{(K)}_j$. Now, define $\tilde{V}=[p^{(2)},...,p^{(K)}]$ and $G(p^{(1)},\tilde{V})=\sum_{j=1}^n g\left(p^{(1)}_j,...,p^{(K)}_j\right)$. Then, $G(p^{(1)},\tilde{V})$ is convex in $\tilde{V}$, as a sum of convex functions. Formally, let $\tilde{V}_1=[p^{(2)},...,p^{(K)}]$ and $\tilde{V}_2=[q^{(2)},...,q^{(K)}]$ such that $\tilde{V}_1,\tilde{V}_2\in \mathcal{V}_{K-1,m}$. Then, for every $\lambda \in [0,1]$,
    \begin{align}
        &G(p^{(1)},\lambda \tilde{V}_1+(1-\lambda) \tilde{V}_2)=\\\nonumber
        &\sum_{j=1}^m \sqrt{\left((1-w_1)p_j^{(1)}-\sum_{k=2}^K w_k\left(\lambda p^{(k)}_j+(1-\lambda)q^{(k)}_j\right)\right)^2+\frac{1}{n}\left( w_1^2p_j^{(1)}+ w_2^2\right)}\leq\\\nonumber
        &\lambda\sum_{j=1}^m \sqrt{\left((1-w_1)p_j^{(1)}-\sum_{k=2}^K w_k p^{(k)}_j\right)^2+\frac{1}{n}\left( w_1^2p_j^{(1)}+ w_2^2\right)}+\\\nonumber
        &(1-\lambda)\sum_{j=1}^m \sqrt{\left((1-w_1)p_j^{(1)}-\sum_{k=2}^K w_k q^{(k)}_j\right)^2+\frac{1}{n}\left( w_1^2p_j^{(1)}+ w_2^2\right)}=\\\nonumber
        &\lambda G(p^{(1)}, \tilde{V}_1)+(1-\lambda)G(p^{(1)}, \tilde{V}_2)
    \end{align}
where the inequality follows from the convexity of $g\left(p^{(2)}_j,...,p^{(K)}_j\right)$.
Finally, Birkhoff–von Neumann theorem states that every doubly stochastic real matrix $U \in \mathcal{U}_m$ is a convex combination of permutation matrices, $T\in \mathcal{T}_m$ with the permutation matrices $T$ being precisely the extreme points (the vertices) of the Birkhoff polytope $\mathcal{T}_m$ \citep{jurkat1967term}. In other words, every doubly stochastic matrix $U \in \mathcal{U}_m$ may be represented as a convex combination of the permutation matrices $T\in\mathcal{T}_m$.
This means that every $\tilde{V}\in \mathcal{V}_{K-1,m}$ is also a convex combination of partial permutation matrices $S\in \mathcal{S}_{K-1,m}$, following the construction of $\tilde{V}$. That is, assume that $\tilde{V}$ is defined by $K-1$ columns of a doubly stochastic matrix $U$. By the Birkhoff–von Neumann theorem, $U=\sum_i \lambda_i T_i$ is a convex combination of $T_i\in \mathcal{T}_m$. This means that $\tilde{V}=\sum_i \lambda_i S_i$ is a convex combination of the corresponding $S_i \in \mathcal{S}_{K-1,m}$. By the convexity of $G(p^{(1)},\tilde{V})$ with respect to $\tilde{V}$ we attain
$$\max_{p^{(2)}...p^{(K)}\in \mathcal{V}_{K-1,m}} \sum_{j=1}^mf\left(p^{(1)}_j,...,p^{(K)}_j\right)\leq\max_{p^{(2)}...p^{(K)}\in \mathcal{S}_{K-1,m}} \sum_{j=1}^mg\left(p^{(1)}_j,...,p^{(K)}_j\right)$$
\end{proof}
Plugging Proposition \ref{prop2} to (\ref{T3_2}) yields 
\begin{align}\label{T3_5}
\mathbb{E}D_{\text{TV}}\left(p,\hat{p}^{w}\right)\leq\max_{p^{(1)}\in\Delta_{m}}\left(\max_{p^{(2)}...p^{(K)}\in \mathcal{S}_{K-1,m}} \sum_{j=1}^mg\left(p^{(1)}_j,...,p^{(K)}_j\right)\right)
\end{align}
Let us now focus on the maximization with respect to $p^{(1)}$. Proposition \ref{prop3} below shows that under the stated condition, $\sum_{j=1}^m g\left(p^{(1)}_j,...,p^{(K)}_j\right)$ is concave in $p^{(1)}$ for every $p^{(2)},...,p^{(K)} \in \mathcal{S}_{K-1,m}$.

\begin{proposition}\label{prop3}
    Assume $w_1\geq \frac{8n-\sqrt{8n}}{8n-1}.$ Then, $\sum_{j=1}^m g\left(p^{(1)}_j,...,p^{(K)}_j\right)$ is concave in $p^{(1)}$ for every $p^{(2)},...,p^{(K)} \in \mathcal{S}_{K-1,m}$.
\end{proposition}
\begin{proof}
    Define $z(t)=\sqrt{(at-b)^2+ct+d}$. The function $z(t)$ is concave in $t$ if ${d^2z(t)}/{dt^2}\leq 0$. Hence, we require that 
    $$2a^2\left((at-b)^2+ct+d\right)-\frac{1}{2}(2a(at-b)+c)^2\leq 0,$$
    which holds if and only if $2abc+2a^2d-c^2/2\leq0.$
    Considering $$g\left(p^{(1)}_j,...,p^{(K)}_j\right)=\sqrt{\left((1-w_1)p^{(1)}_j-\sum_{k=2}^K w_kp^{(k)}_j\right)^2+\frac{1}{n}\left( w_1^2p_j^{(1)}+ w_2^2\right)}$$ with respect to $p_j^{(1)}$ we have that $a=(1-w_1)$, $b=\sum_{k=2}^Kw_kp_j^{(k)}$, $c=w_1^2/n$ and $d=w_2^2/n$.
    Therefore, we require \begin{align}\label{prop4_1}
    2(1-w_1)\frac{w_1^2}{n}\sum_{k=2}^Kw_kp_j^{(k)}+2(1-w_1)^2\frac{w_2^2}{n}-\frac{w_1^4}{2n^2}\leq 0
    \end{align}
    for the desired concavity of $\sum_{j=1}^m g\left(p^{(1)}_j,...,p^{(K)}_j\right)$. 
    Notice that $w_k\leq w_2 \leq 1-w_1$ for every $k=2,...,K$, where the last inequality follows from $\sum_{k=1}^Kw_k=1$. Further, $\sum_{k=2}^Kw_kp_j^{(k)}\leq w_2$ since every $p^{(2)},....,p^{(K)}\in S_{K-1,m}$ satisfies $p^{(k)}_j\in\{0,1\}$ and $\sum_{k=2}^Kp^{(k)}_j\leq 1$. Therefore, it is enough to show that 
    \begin{align}\label{prop4_2}
    \frac{2(1-w_1)^2w_1^2}{n}+\frac{2(1-w_1)^4}{n}-\frac{w_1^4}{2n^2}\leq \frac{4(1-w_1)^2w_1^2}{n}-\frac{w_1^4}{2n^2}\leq 0.
    \end{align}
    Solving the above quadratic inequality yields $w_1\geq \frac{8n-\sqrt{8n}}{8n-1}.$ 
   
\end{proof}

Let us now apply Proposition \ref{prop3} to (\ref{T3_5}). Denote 
\begin{align}
q^{(2)},...,q^{(K)}=\underset{p^{(2)}...p^{(K)}\in \mathcal{S}_{K-1,m}}{\text{argmax}} \sum_{j=1}^mg\left(p^{(1)}_j,...,p^{(K)}_j\right).
\end{align}
Recall that $q^{(k)}_j\in\{0,1\}$ for all $j=1,...,m$ and $k=2,..,K$. In addition, $\sum_{k=2}^Kq^{(k)}_j \in \{0,1\}$ for every $j=1,...,m$  and $\sum_{j=1}^m q^{(k)}_j= 1$ for every $k=2,..,K$. Denote by $\mathcal{J}$ the collection of $j$ indices for which $\sum_{k=2}^Kq^{(k)}_j=1$. Notice that only a single element in this sum equals one, while the others are zero. We denote the
index of this element by $k_j$. That is, for every $j \in \mathcal{J}$, we have $q^{(k_j)}_j=1$ while $q^{(k)}_j=0$ for every $k\neq k_j$. Therefore,

\begin{align}\nonumber
    \sum_{j=1}^mg\left(p^{(1)}_j,q^{(2)}_j,...,q^{(K)}_j\right)=&\sum_{j\in \mathcal{J}}g\left(p^{(1)}_j,q^{(2)}_j...,q^{(K)}_j\right)+\sum_{j\notin \mathcal{J}}g\left(p^{(1)}_j,q^{(2)}_j,...,q^{(K)}_j\right)=\\\nonumber
    &\sum_{j\in \mathcal{J}} \sqrt{\left((1-w_1)p^{(1)}_j-w_{k_j}\right)^2+\frac{1}{n}\left( w_1^2p_j^{(1)}+ w_2^2\right)}+\\\nonumber
    &\sum_{j\notin \mathcal{J}} \sqrt{(1-w_1)^2\left(p^{(1)}_j\right)^2+\frac{1}{n}\left( w_1^2p_j^{(1)}+ w_2^2\right)}.
\end{align}
Let us first focus on the second summation. Notice that all summands are identical functions of $p_j^{(1)}$. In addition, the sum is concave with respect to $p_j^{(1)}$, as suggested in Proposition \ref{prop3}. 
Denote $\pi=\sum_{j\notin \mathcal{J}} p_j^{(1)}$. By the concavity and the symmetry of the second summand, we have that
\begin{align}\label{asda}
    &\sum_{j\notin \mathcal{J}} \sqrt{(1-w_1)^2\left(p^{(1)}_j\right)^2+\frac{1}{n}\left( w_1^2p_j^{(1)}+ w_2^2\right)}\leq\\\nonumber
    &\sum_{j\notin \mathcal{J}} \sqrt{(1-w_1)^2\left(\frac{\pi}{m-|\mathcal{J}|}\right)^2+\frac{1}{n}\left( w_1^2\frac{\pi}{m-|\mathcal{J}|}+ w_2^2\right)}=\\\nonumber
    &\sqrt{(1-w_1)^2\pi^2+\frac{(m-|\mathcal{J}|)}{n}\left( w_1^2\pi+ w_2^2(m-|\mathcal{J}|)\right)}.
\end{align}
That is, a uniform distribution (over a total mass of $\pi$) maximizes the summation above. For simplicity of presentation (and without loss of generality), we denote $j\in\mathcal{J}$ as $j=1,...,J$ where $J=|\mathcal{J}|$. Notice that from maximization considerations, $J=K-1$. That is, maximizing $\sum_{j=1}^mg\left(p^{(1)}_j,p^{(2)}_j,...,p^{(K)}_j\right)$ over $\tilde{V}=[p^{(2)},...,p^{(K)}] \in \mathcal{S}_{K-1,m}$ necessarily results in a maximal possible number of rows in $\tilde{V}$ for which the sum of the entire row is one. 
Hence, we have that 
\begin{align}\label{almost_done}
&\max_{p^{(1)}_1,...,p^{(1)}_{K-1}}\sum_{j=1}^mg\left(p^{(1)}_j,q^{(2)}_j,...,q^{(K)}_j\right)=\\\nonumber
&\max_{p^{(1)}_1,...,p^{(1)}_{K-1}} \sum_{j=1}^{K-1} \sqrt{\left((1-w_1)p^{(1)}_j-w_{k_j}\right)^2+\frac{1}{n}\left( w_1^2p_j^{(1)}+ w_2^2\right)}+\\\nonumber
    &\quad\quad\quad\quad\sqrt{(1-w_1)^2\pi^2+\frac{(m-K+1)}{n}\left( w_1^2\pi+ w_2^2(m-K+1)\right)}
\end{align}
where $\pi=1-\sum_{j'=1}^{K-1}p_{j'}^{(1)}$. Importantly, (\ref{almost_done}) is a concave maximization problem over a simplex of $K$ variables, $p^{(1)}_1,...,p^{(1)}_{K-1}, \pi$, due to Proposition \ref{prop3}.

Finally, notice that the (\ref{almost_done}) depends on   $q^{(2)},...,q^{(K)}\in\mathcal{S}_{K-1,m}$ only through $w_{k_j}$. That is, every choice of $p^{(2)},...,p^{(K)}\in\mathcal{S}_{K-1,m}$ corresponds to a different set of $w_{k_j}$. However, $p^{(2)},...,p^{(K)}\in\mathcal{S}_{K-1,m}$ are degenerate and unique distributions. This means that $w_{k_j}$ are unique values from the given $w_2,...,w_K$. In other words, the choice of $p^{(2)},...,p^{(K)}\in\mathcal{S}_{K-1,m}$ only dictates a bijection between $\{w_{k_j}\}_{j=1}^{K-1}$ and $\{w_j\}_{j=2}^{K}$.
Therefore, every choice $p^{(2)},...,p^{(K)}\in\mathcal{S}_{K-1,m}$ leads to exactly the same maximization problem over $p^{(1)}_1,...,p^{(1)}_{K-1}$. For simplicity, and without loss of generality, we choose $w_{k_j}=w_{j+1}$ for $j=1,...,K-1$ which concludes the proof.

\section{Conclusions}

Given a set of algorithms and a benchmark of datasets, it is not entirely clear how to identify the best-performing algorithm or even how to establish a meaningful ranking among them. In this work we claim that a natural evaluation criterion is the probability that an algorithm will win (rank highest) on a future, unseen, dataset.  We introduce a novel conceptual framework for estimating this win probability, which generalizes the ML approach. Specifically, instead of simply counting the number of past wins, we also account for the number of times each algorithm finished second, third and so forth. This way, we provide a ``natural'' weighting for the performance of the algorithms on the benchmark set. We present two evaluation schemes within this framework. The first involves data-independent weights, predetermined before the evaluation begins, which are optimized to minimize worst-case risk. The second is a more adaptive, data-dependent approach based on a leave-one-out strategy, which shows significant improvements over existing alternatives. We validate both schemes through synthetic and real-world experiments.

Our analysis yields several important insights for comparative studies. First,  introducing too many competing algorithms is counter-productive. Specifically, by doing so we inherently add unnecessary noise to the rankings, which makes it more difficult to conclude on the winner. We advocate for the exclusion of algorithms that are not serious contenders, as their inclusion merely inflates the scope without improving the quality of the evaluation. Second, we argue that focusing solely on winners is insufficient. Lower-ranked placements contain valuable information, and relying on the MLE is suboptimal. However, our method converges to the MLE when the number of datasets is large or when there are highly dominant algorithms, in which case win-counting suffices.

Our proposed framework is not limited to algorithm comparison. It generalizes to any scenario involving repeated competitions among multiple entities. For instance, consider a horse race with $m$ horses, where past performance data is available and spectators aim to predict the most likely winner of a future race. Here, the horses correspond to the algorithms and their performance in past races are their benchmark results. Naturally, the goal is to determine the horse that is most likely to win a future race. We consider such applications and their natural adaptions for our future work.

\section*{Acknowledgements}
A.P. is supported in part by the Israel Science Foundation
(grant No. 963/21). To author would like to thank  Fern{\'a}ndez-Delgado  \citep{fernandez2014we}, Shmuel \citep{shmuel2024comprehensive} and McElfresh \citep{mcelfresh2024neural} for providing their raw experimental data for this study. 

\bibliographystyle{plain}
\bibliography{bibi}  
\end{document}